\documentclass[twoside,11pt]{article}

\usepackage{jmlr2e}

\usepackage{xcolor}
\usepackage{subcaption}
\usepackage{multirow}
\usepackage[cmex10]{amsmath}
\usepackage{relsize}
\usepackage{mleftright}

\newtheorem{mytheorem}{Theorem}[section]
\newtheorem{mycorollary}{Corollary}[section]
\newtheorem{myclaim}{Claim}[section]
\usepackage{hyperref}

\ShortHeadings{Learning Activation Functions: A new paradigm of understanding Neural Networks}{SLNN}
\firstpageno{1}

\begin{document}

\title{Learning Activation Functions: A new paradigm of understanding Neural Networks}

\author{\name Mohit Goyal \email goyal.mohit999@gmail.com \\
       \addr Electrical and Computer Engineering Department\\
       University of Illinois at Urbana Champaign\\
       Illinois, USA
       \AND
       \name Rajan Goyal \email rajangoyal910@gmail.com \\
       \addr Department of Electrical Engineering\\
       Indian Institute of Technology Delhi\\
      New Delhi, India
       \AND
       \name Brejesh Lall \email brejesh@ee.iitd.ac.in \\
       \addr Department of Electrical Engineering\\
       Indian Institute of Technology Delhi\\
      New Delhi, India}

\maketitle

\begin{abstract}
There has been limited research in the domain of activation functions, most of which has focused on improving the ease of optimization of neural networks (NNs). However, to develop a deeper understanding of deep learning, it becomes important to look at the non linear component of NNs more carefully. In this paper, we aim to provide a generic form of activation function along with appropriate mathematical grounding so as to allow for insights into the working of NNs in future. We propose ``Self-Learnable Activation Functions'' (SLAF), which are learned during training and are capable of approximating most of the existing activation functions. SLAF is given as a weighted sum of pre-defined basis elements which can serve for a good approximation of the optimal activation function. The coefficients for these basis elements allow a search in the entire space of continuous functions (consisting of all the conventional activations). We propose various training routines which can be used to achieve performance with SLAF equipped neural networks (SLNNs). We prove that SLNNs can approximate any neural network with lipschitz continuous activations, to any arbitrary error highlighting their capacity and possible equivalence with standard NNs. Also, SLNNs can be completely represented as a collections of finite degree polynomial upto the very last layer obviating several hyper parameters like width and depth. Since the optimization of SLNNs is still a challenge, we show that using SLAF along with standard activations (like ReLU) can provide performance improvements with only a small increase in number of parameters.
\end{abstract}

\noindent
An extension of this work has been published at IEEE WCCI 2020 and can be accessed \href{https://ieeexplore.ieee.org/abstract/document/9207535}{https://ieeexplore.ieee.org/abstract/document/9207535}. Our code is publicly available at \href{https://github.com/mohit1997/PAF}{https://github.com/mohit1997/PAF}.

\section{Introduction}


Better architectures, newer activations and faster optimization techniques being proposed everyday have fueled the huge success of deep learning. 
Most of the past work on activations has been with the goal of achieving superior empirical performance with justifications that are more intuitive, lacking proper mathematical investigation. As far as the definition of activation functions is concerned, universal function approximation theorem (UFAT) lays out the most widely accepted one. It defines it as a \emph{``non constant, bounded and continuous function''}. While most of the activations conform to this criterion, NNs with linear units (unbounded) and its variants, still satisfy UFAT (\cite{unboundedproof}). The central goal in most of the application oriented research invloving neural networks (NNs) is to get higher performance which can be in terms of more complex underlying functions or NNs' better generalization over unseen data. The performance of NNs is heavily affected by the optimization method---stochastic gradient descent (SGD, \cite{sgd}), used for training them. It gives rise to a set of desirable properties for the activations which can potentially improve training time or guarantee convergence. Monotonic variation, zero centred nature, and appropriate gradient range are some of their widely studied properties. In this paper, we take a different route with objective of viewing activations from a theoretical perspective and characterize their role in NNs. 

 While employing NN models, it is indeed valid to ask which activation function would perform the best. Deep learning community has attempted to answer this problem in variety of ways, some people analyze the gradients of activations addressing the vanishing and exploding gradient problems. Some also try to study the variance arising in NNs due to the shape of activation functions, for example, ELU (\cite{elu}) saturates for the negative input values which contributes to its robustness to noisy inputs. A more rigorous way could be to do a grid search over all possible activations. \cite{swish} uses automated reinforcement learning based search on composite combination of existing activation functions. However, the resulting space of functions undergoing search is finite and therefore small. Here we want to develop analytic and more insightful ways to answer this question. We aim to formalize this concept of activation function without any prior assumptions or borrowing motivation from biology. Finally, what we want is a generalized form of activation function, spanning a much bigger space. Lastly, there remains one more factor which is relatively sparsely studied and needs to be accounted for---nature of the task and training data.
Ideally the NN should be able to incorporate the variability arising in a learning problem due to the type of task and properties of data via adaptation. The same should hold for the activation, but evidently not all perform equivalently on any of the task. Therefore, we propose approximating the optimal activation function (unique for each task and each data set) by assuming a functional form at each neuron with learnable parameters which are updated during training, thereby entering the domain of adaptive activation functions (AAFs). \cite{prelu, maxout, NinN, apla} represent some of the important milestones in the past research on AAFs. \cite{SAAF} proposes a piece-wise defined polynomial activation function specifically for regression tasks. It also highlights difficulties experienced while training NNs with AAF. If the AAF is too simplistic, it might not learn a good approximation of the optimal activation function (the one with smallest training error). On the other hand, a highly flexible activation function, with many parameters, will result in severe over-fitting. In light of the above issues, a major contribution of this paper is (i) a working adaptive activation function called \emph{`SLAF'} motivated from polynomial approximation of univariate continuous functions. With respect to optimization, it is difficult to train AAFs with non monotoic behaviour (\cite{tamingsine}). Hence, we also provide a model setup for SLAF which allows SLNNs to be trained with SGD, along with training routines for practical purposes. While we achieve similar performances on many of the tasks, and outperform existing activations on synthetic datasets, the contributions hold importance in an analytic sense. We provide interesting insights into the nature of neural networks activated with SLAF (SLNNs) along with (ii) mathematical bounds on number of parameters needed to represent an SLNN as a function of their degree (Theorem \ref{th: representaion of nn}). Since, SLNNs can be approximated to any other NN activated with conventional activation functions, characteristics of conventional NNs can be related to those of SLNNs'. We emphasize on the multilayered architecture of NNs (SLAF) by various experiments and provide theoretical grounds for our inferences. Importantly, we assume no constraints on our activation function except continuity and differentiability for generic characterization of activation function. Note: We do not claim UFAT holds for SLNNs, which theoretically is not disadvantageous. We show experiments on standard datasets to demonstrate the approximation capabilities of SLNNs.


\section{Motivation and Model Setup}

This section describes the idea of defining best learnable activation as a good approximation of the optimal activation function on a predefined basis. Consider a neural network model with learnable activation F and for simplicity, let's assume that \(\hat{F}\) be its optimal function for the specific task and data distribution. \(\Tilde{F}\) denotes the projection of \(\hat{F}\) on a defined set of basis functions, with $N$ elements, represented as $ \{\phi_0, \phi_1, \ldots, \phi_{N-1}, \} $  (eq. \eqref{approximation equation}). Hence, for a fixed basis the whole problem of learning \(F\) boils down to learning the right set of coefficients of basis elements which post training should ideally converge to $\{a_0, a_1, \ldots, a_{N-1}\}$. 
\begin{equation} \label{approximation equation}
\hat{F} \approx a_0 \phi _0 + a_1 \phi _1 + \ldots + a_{N-1} \phi _{N-1} = \Tilde{F}
\end{equation}

\begin{figure}[!htb]
    \centering
    \includegraphics[width=0.7\textwidth]{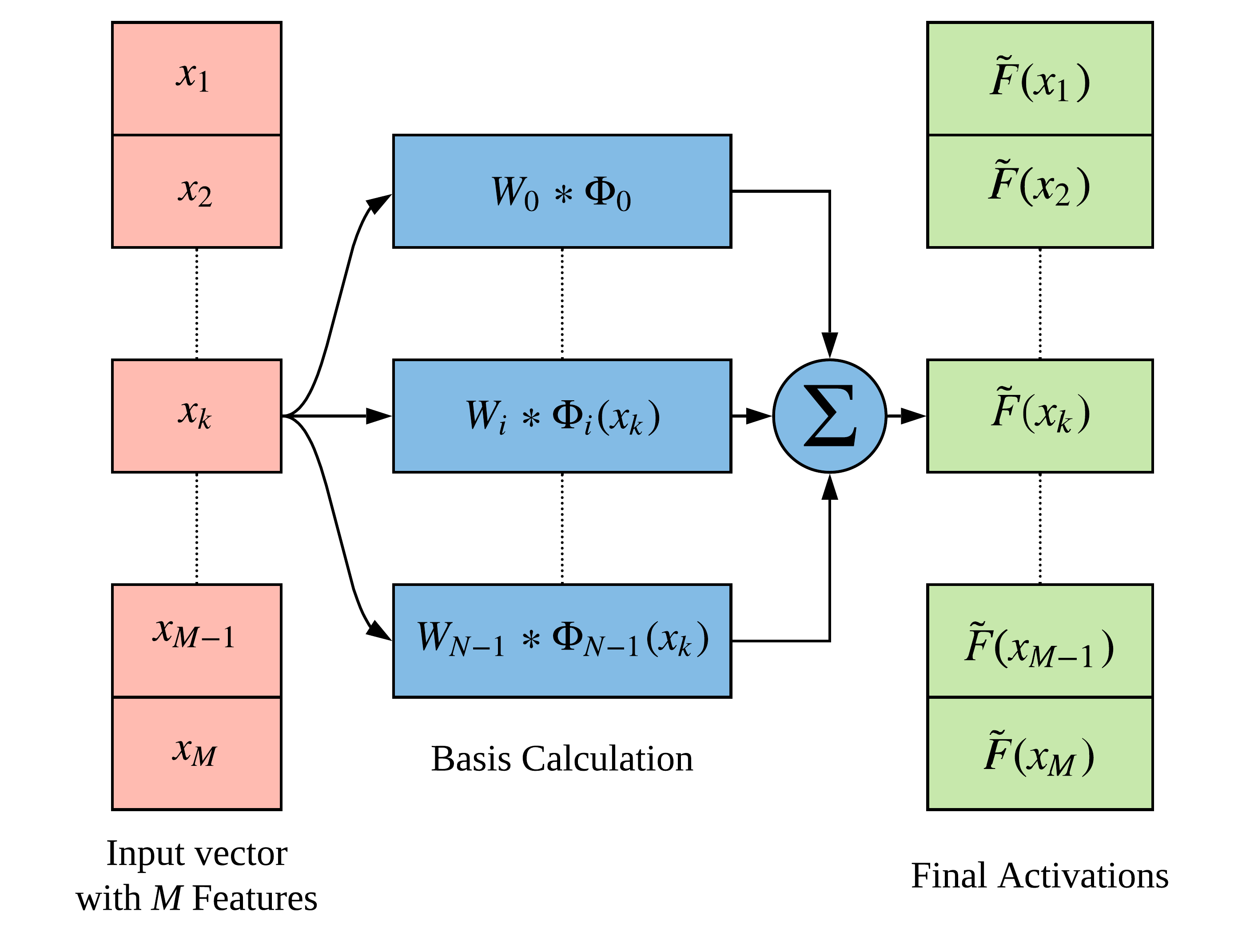}
    \caption{Illustration for a hidden layer activated with SLAF using $N$ basis elements. \(W_i\)\textit{s} are learnable parameters.}
        \label{fig:SLAFMODEL}
\end{figure}

Since the performance of resulting NN is contingent on how close \(\Tilde{F}\) and \(\hat{F}\) are, one would like to select a basis which can provide a good approximation of \(\hat{F}\) using only fewer number of basis elements, and thus reducing the computational complexity. Moreover, the goodness of basis will also depend upon how easily the SLNN can be optimized. In our work, we explore three different basis, viz. (i) Even Mirror Fourier (EMFN) Basis (\cite{emfn}) with truncated input, (ii) Taylor Polynomial Basis, and (iii) Chebyshev Basis. To compare these basis, we approximate existing activation functions ReLU, Sigmoid and tanh. It turns out that both EMFN basis and Taylor polynomial basis provide good enough approximation of existing activation functions and hence SLNN using either Taylor or EMFN basis would approximate these activations, if they were the optimal ones. The functional form of EMFN basis elements contain sinusoidal terms, thereby making it difficult to train using gradient descent (\cite{tamingsine}). Taylor polynomials have easier analytic form and its basis elements can be computed in polynomial time. Henceforth, we will only use Taylor approximation for all experimental and theoretical purposes. Figure \ref{fig:SLAFMODEL} shows how activations are calculated with SLAF on a hidden layer of standard NN.

\subsection{Using Taylor Polynomial Basis}
Although Taylor basis puts no restriction on the range of input \(x\), it can not be directly employed in a NN being optimized through SGD, because of the nature of gradients corresponding to each basis element. SLAF \(f(x)\) and its gradients can be written as
\begin{gather} \label{eq:taylor approximation 2}
    f(x) = a_0 + a_1x + a_2x^2 + \ldots + a_{N-1}x^{N-1}\\
    \nabla_x{f(x)} = a_1 + 2a_2x + \ldots + (N-1)a_{N-1}x^{N-2}, \; \frac{\partial f}{\partial a_i}(x) = x^i \label{eq:taylor gradient}
\end{gather}
The proportionality of gradient to \(x\) and its powers can lead to problems of exploding and vanishing gradients as the scale of input changes. This effect is more pronounced as the depth or the degree of SLNN increases. To handle this issue, we perform mean variance normalization on each basis function. The transformed basis functions \(\hat{x_i}\) are then used in SLAF:
\begin{gather}\label{eq:SLAFtaylornorm}
    \hat{x_i} = \frac{x^i - mean(x^i)}{\sqrt{var(x^i) + \epsilon}},\hspace{3mm} 
    i \in \{0, 1, 2, \ldots, k \}\\\label{eq:final taylor eq}
    f(x) = a_0' + a_1'\hat{x} + \ldots + a_{N-1}'\hat{x}^{N-1} = a_0 + a_1x + \ldots + a_{N-1}x^{N-1}
\end{gather}
where mean and variance are computed over the training data set. The coefficients ($a_i'$'s) can serve as means for recovering original mean and variance resulting in information preservation. The technique corrects the scale of very large or very small value of basis functions and helps in faster convergence (\cite{effbpp}). For CNN, \(x_i\)'s would denote a channel rather than one feature. We store the exponentially averaged means and variances calculated from training data statistics for performing normalization at test time similar to batch normalization (\cite{bnorm}).

\section{Representation of SLNN}\label{sec: representation}
It is first important to understand the mathematical proprieties of SLNNs in reference to normal NNs activated with generic activation functions. This allows us to quantify the theoretical differences between the learning capacity of both kinds of networks. We prove that for any neural network having only Lipschitz continuous activation function, there exists a corresponding SLNN which can completely approximate (each layer) the former to any arbitrarily low error (refer to appendix \ref{th: nn equivalence}). The proof in the appendix shows that for a fixed degree of SLAFs used in SLNN, the maximum approximation error would be directly proportional upon the magnitude of weights and the width of each layer. This explain two properties of neural networks, (i) increasing the width allows NNs to learn more complex tasks, (ii) L2 regularization which reduces the magnitude of parameters reduces the complexity of NN allowing for better generalization. More formally, higher magnitude of weights and larger width of layers, both map to higher degrees of SLAFs in the corresponding SLNN to maintain a fixed approximation error.

\noindent
One of the interesting outcome of employing polynomial activation is the resulting elegant polynomial representation of SLNN. Specifically, a fully connected NN activated with SLAF can be completely expressed as a collection of multivariate polynomial of input features up to the penultimate layer. The same representation would hold for the complete NN if the final activation is affine (regression tasks), however for classification tasks where the final activation is a sigmoid or a softmax, the claims on representation are valid only up to the second last layer of NN.

\textbf{Definitions:}
Consider a set \(\textbf{x} = \{x_1, ..., x_n\} \subseteq \mathbb{R}^n\). Let \(\mathcal{B}_k(\textbf{x})\) denote the set of elements of polynomial basis with degree \(k\) constructed using elements of set \(\textbf{x}\). Then,
\begin{equation*}
    \mathcal{B}_k(\textbf{x}) = \left\{\prod_{i=0}^{i=n}x_i^{\alpha_i}\; \middle|\ \; \sum_{i=1}^{i=n}\alpha_i = j, \:j \in \{0, 1, ..., k\} \right\}
\end{equation*}
We call \(\mathcal{B}_k(\textbf{x})\) as the polynomial basis defined over set \(\textbf{x}\) with degree \(k\).

\noindent
\emph{For the sake of readability and clarity, the proofs of the following theorems have been shifted to the appendix.}

\begin{theorem}\label{th: cardinality of basis}
The cardinality of \(\mathcal{B}_k(\textbf{\emph{x}})\), the set of elements of polynomial basis with degree \(k\) constructed using elements of set \(\textbf{x}\), denoted by \(N_{\mathcal{B}_k}\)  is equal to \(\binom{k + n}{n}\), where \(n\) is the cardinality  of set \(\textbf{\emph{x}}\).
\end{theorem}



Let \(\textbf{X}^{\mathcal{B}_k}_{N_{\mathcal{B}_k}\times 1}\) denote the \(N_{\mathcal{B}_k}\times1\) matrix containing the elements of set \(\mathcal{B}_k(\textbf{x})\). Now we show that any general neural network having a structure as mentioned below can be completely represented in a polynomial form.

\begin{theorem}\label{th: representaion of nn}

Consider an SLNN with $H$ hidden layers and input denoted as $X_{n \times 1}$ and output as $\textbf{Y}_{m \times 1}$. If the activation at the final layer is linear and all the hidden layers are activated with SLAF of degree $k_i$, where $i$ is the index of the hidden layer. Then, the output of this NN can be reparametrized and written as 
\begin{equation}
    \textbf{Y}_{m\times 1} = \textbf{W}_{m\times N_{\mathcal{B}_k}}\textbf{X}^{\mathcal{B}_k}_{N_{\mathcal{B}_k}\times 1}
\end{equation}
where, \(k = \prod_{i=1}^{H}k_i\), called as degree of SLNN, \(\textbf{W}_{m \times N_{\mathcal{B}_k}}\) are the new parameters and \(\textbf{X}^{\mathcal{B}_k}_{N_{\mathcal{B}_k}\times 1}\) is the vector containing polynomial features of degree $\leq k$ in $X_{n \times 1}$. The subscripts in the notation denote matrix size. \\ \textbf{Note:} SLAF with degree equal to one is equivalent to linear/no activation. Therefore, this result directly holds for regression tasks.
\end{theorem}

Theorem \ref{th: representaion of nn} shows that the output of SLNN can be easily represented as a collection of polynomials with degree defined by SLAF applied at each layer of  NN. Note that the learnable parameters corresponding to all the linear operations and activation are absorbed into one single matrix given by \(\textbf{W}_{m \times N_{\mathcal{B}_k}}\). As a result, theoretically training a SLNN becomes equivalent to finding out the optimal value of the new weight matrix. Although we do not provide a generalization of the above theorem for classification tasks, it remains valid up to the penultimate layer of a SLNN (with softmax or sigmoid at the final layer). It is easy to see that the resulting reparametrized form would be equivalent to a linear classifier over \(\textbf{X}^{\mathcal{B}_k}_{N_{\mathcal{B}_k}\times 1}\) and polynomial regression for classification and regression tasks respectively. The direct result of this reparametrization is a hard upper bound on total number of variables that the SLNN will have, given by $m\cdot N_{\mathcal{B}_k}$, where $m$ is one for regression or the number of classes for classification. One should be careful that this bound is only due to the theoretical equivalence and ignores the effect of optimization algorithm. \emph{Different parametrization (\cite{weightshare}), normalization techniques (\cite{bnorm}) can hugely affect the empirical performance of NNs (SLNNs here).}

Though it might seem that SLNNs are redundant for regression tasks and one would prefer performing polynomial regression which is simpler and much more efficient, the scalability issues severe its compatibilty with high dimensional data. As the degree increases, their is a surge in number of features which bottlenecks its practical applicability. Fortunately this is much simpler in SLNNs, where the polynomial representation is implicitly learned relaxing computational issues with memory limitations. 

\section{Experiments}\label{sec: experiments}
To demonstrate the effectiveness of SLNN and gauge its performance with BP algorithm, we perform experiments on regression, classification and learning sparse polynomials. In all the experiments, we apply L2 regularization penalty on the activation coefficients.


\subsection{SLNN as Polynomial regression}
The section empirically validates the claim of theorem \ref{th: representaion of nn} on both regression and classification tasks. In addition, it also shows the effectiveness of SLNN in learning sparse polynomials. 
\begin{enumerate}
    \item \textbf{Regression - Boston Housing:} In a regression setting, an SLNN can be completely reduced to a polynomial. Hence, it becomes equivalent to a linear regression applied on the polynomial basis of original input features. We compare the performance of four algorithms to verify the empirical validity of the above statement taking into account the optimization method as well. For this experiment, we take boston housing dataset, which has thirteen features originally and the output is to predict the house price. We test the following algorithms: (i) Standard Neural Network with two hidden layers each activated with ReLU and accelerated with batch normalization \emph{(NNRELUBN)} along with L2 regularization (ii) SLNN  with two hidden layers each activated with SLAF of degree four and optimized with SGD \emph{(SLNN)} (iii) Linear Regression on Polynomial features with SGD \emph{(LRSGD)} with L1 regularization (iv) Lasso Linear Regression \emph{(LLS)} optimized with coordinate descent.
    
    \begin{table}[!htbp]
    \centering
\begin{tabular}{|c|c|c|c|}
\hline
\textbf{Algorithm} & \textbf{Degree/Description} & \textbf{Training RMSE} & \textbf{Testing RMSE} \\ \hline
\textit{NNRELUBN}  & 2 Hidden Layers             & 1.32                  & 3.78                  \\ \hline
\textit{SLNN}      & $k=8$, $k_1=4$, $k_2=2$     & 2.09                   & 3.98                  \\ \hline
\textit{LRSGD}     & Degree=8 , Penalty=0.01                   & 22.03                  & 22.69                 \\ \hline
\textit{LLS}       & Degree=8, Penalty=0.01                    & 1.59                   & 3.06                  \\ \hline
\end{tabular}
\caption{Comparision of the four algorithms on boston housing dataset. NNRELUBN with model similar to SLNN provides a baseline for comparison with other methods. Note: We use Adam optimizer in place of the vanilla SGD for optimization. \emph{RMSE stands for root mean squared error.}}
\end{table}
    As a result of theorem \ref{th: representaion of nn}, in all three methods \emph{SLNN, LRSGD,} and \emph{LLS}, the underlying representation is same. However, the performance of LRSGD is quite poor as compared to the other two algorithms. Being theoretically same, the global minima of LRSGD is the same as of the other two.
    The reason is the sub-optimality of SGD in converging to the global minima. Since, SLNN has more number of parameters than those required for its polynomial representation, the number of global minimas are more in space of learnable parameters. While on the other hand, LRSGD has exactly the same number of parameters needed for its representation. Due to this redundancy in the network parameters (\cite{globalminima}), SLNN tends to easily converge with SGD and exhibit performance similar to LLS (derivative-free optimization method). One should note that both LLS and LRSGD are not scalable unlike SLNN with higher dimensional inputs (Theorem \ref{th: cardinality of basis}).
    
    \item \textbf{Classification - Two Spiral:} In the previous case, we observed that LLS and SLNN perform similar in terms of the test error, where LLS converges much faster. However, for classification tasks, we observe that logistic regression over the polynomial basis combined with other optimization methods (\cite{sklearn}) doesn't turn out to be as beneficial. At the same time, SLNN performs significantly better and converges most of the times. Hence, SLNN doesn't limit learning a polynomial feature space even in a classification setting and therefore bcomese advantageous. To demonstrate the same, we employ two spiral classification problem tested with (i) conventional \emph{NN} with batch normalization and two hidden layers with ReLU activation, (ii) \emph{SLNN} with two hidden layers each with SLAF activation of degree seven each, and (iii) Logistic Regression on polynomial basis of degree fourteen with SAGA (\cite{saga}) optimization \emph{(LRSAGA)}. We skip results of other approaches for lack of relevance. Note that we specifically choose a smaller two layered architecture to demonstrate how a higher degree SLAF can compensate for deeper/wider NN with the same number of parameters. Training and Test data-set are synthetically generated and randomly chosen. Figure \ref{fig: tsneclass} and \ref{fig: tsneNN} shows the classification boundary learnt by SLNN and NN, respectively. We can see that due to the underlying assumptions on the activation function's differentiability, the resulting boundary learned is itself smooth and provides good generalization (extending both the spirals will decrease classification accuracy in the case of NN). However, the boundary learnt by NN with ReLU activation displays sharp turns and aesthetically unpleasing boundary. This is again due to form of ReLU which has a bent at origin.
  \begin{figure}[!htbp]
    \centering
    \includegraphics[width=0.45\textwidth]{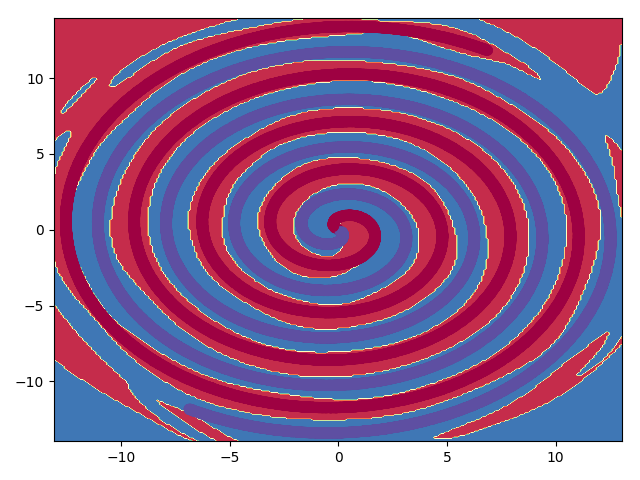}
    \caption{SLNN's ($k=14$) classification map}
    \label{fig: tsneclass}
  \end{figure}
  \begin{figure}[!htbp]
  \centering
    \includegraphics[width=0.45\textwidth]{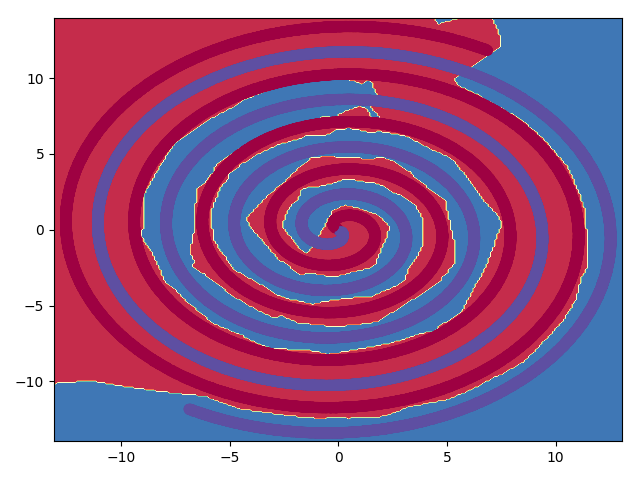}
    \caption{Standard NN's classification}
    \label{fig: tsneNN}
  \end{figure}

\begin{table}[!htbp]
\centering
\begin{tabular}{|c|c|c|}
\hline
\textbf{Algorithm} & \textbf{Train Accuracy (\%)} & \textbf{Test Accuracy (\%)} \\ \hline
\textit{NN}    & 80.14                  & 80.98                  \\ \hline
\textit{SLNN}  & 99.41                   & 99.69                 \\ \hline
\textit{LRSAGA}  & 74.69                  & 74.60                 \\ \hline
\end{tabular}
\caption{Comparison}
\end{table}

    \item \textbf{Regression - Learning Sparse Polynomial:}
    Now we shift our focus to the task of learning sparse polynomials (which have small number of monomial terms with non zero coefficients). NNs have been theoretically studied to estimate their ability to approximate polynomials \cite{polyapprox}. \cite{nnpolylearn} proves that irrespective of an activation function, a single layered neural network can learn k-sparse polynomial (with k monomial terms) of small degrees in finite iterations with appropriate number hidden nodes. On similar lines we design experiments to show that, practically, the choice of activation affects the generalization of the neural network on unseen data points. We experiment with polynomials of degrees three and four on hundred variables (with standard normal distribution) having ten monomial terms. A three layered architecture will be used for this experiment. NN uses hyperbolic tan as activation, accelerated with batch normalization, and SLNN uses same architecture with SLAF activation. We employ L1 regularization on the first layer to take into account the sparsity for both the models. Note that here we do not use the same activation weights/coefficients for each hidden node across a layer.

\begin{table}[!htbp]
\centering
\label{tab: polyres}
\begin{tabular}{|c|c|c|c|}
\hline
\textbf{Model}                          & \textbf{Degree} & \textbf{Training MSE} & \textbf{Testing MSE} \\ \hline
\multirow{2}{*}{\textit{\textbf{NN(Tanh/ReLU)}}}   & 3               & 0.06/0.03                  & 2.25/0.40                 \\ \cline{2-4} 
                                        & 4               & 0.12/0.12                  & 16.24/12.90                \\ \hline
\multirow{2}{*}{\textit{\textbf{SLNN}}} & 3               & 0.03                  & 0.03                 \\ \cline{2-4} 
                                        & 4               & 0.03                  & 0.03                 \\ \hline
\end{tabular}
\caption{Comparing NN and SLNN on learning polynomials.}
\end{table}

\end{enumerate}
\subsection{Standard Classification Tasks}
Since the basis chosen restricts the subspace of activations to only polynomial approximations of finite degrees, it might seem that the model capacity would be greatly reduced. Even though it is a challenge to optimize SLNNs with SGD, we show that SLNNs can perform considerably similar even with finite degree polynomial representations. In this section, we test and compare the performance of SLNNs on standard classification datasets to NNs activated with ReLU activation. Note that we want to showcase the approximation power of SLNNs, therefore avoid experimenting with other activations. 
\begin{enumerate}
    \item \textbf{MNIST}: MNIST is standard hand digit image classification dataset. We experiment with a custom convolutional neural network (NNRELU) with two convolutional layers involving ReLU activation, batch normalization and maxpooling followed by two fully connected layers, the latter one being a standard softmax layer. We replace all the RelU activations with SLAF and call that model as SLNN. 
    \begin{table}[!htbp]
\centering
\begin{tabular}{|c|c|c|}
\hline
\textbf{Algorithm} & \textbf{Degree} & \textbf{Test Accuracy (\%)} \\ \hline
\textit{NNRELU}    & -                  & 99.34                  \\ \hline
\textit{SLNN}  & $2 \times 3 \times 3$                  & 99.55                 \\ \hline
\end{tabular}
\caption{Comparison on MNIST Dataset}
\end{table}
    \item \textbf{CIFAR-10:} This is another image classification dataset consisting of 60000 images labeled in one of the ten classes. We use Resnet architecture with ReLU activation. We show two variants of Resnet where the activation of first layer is replaced by SLAF of degree two, and another where all the activations are replaced with SLAFs of degree two.
    
    \begin{table}[!htbp]
\centering
\begin{tabular}{|c|c|c|c|c|}
\hline
\textbf{Architecture} & \textbf{\# Layers} & \textbf{\begin{tabular}[c]{@{}l@{}}\# Activation functions \\ replaced by SLAF\end{tabular}} & \textbf{\# Parameter} & \textbf{Error(\%)} \\ \hline
\textbf{ResNet}       & 32               & -                                                                                            & 0.46M                 & 7.51               \\ \hline
\textbf{ResNet}       & 44               & -                                                                                            & 0.66M                 & 7.17               \\ \hline
\textbf{ResNet}       & 32               & 1 (SLAF $k=2$)                                                                              & $\sim$0.46M           & \textbf{7.12}      \\ \hline
\textbf{ResNet}       & 32               & 31 (SLAF $k=2$)                                                                              & $\sim$0.46M           & 8.50               \\ \hline
\end{tabular}
\caption{Testing error on CIFAR-10 using different architectures and activation functions. $k$ is the order of Taylor series used.}
\label{tab:Resnet accuracy}
\end{table}
    \item \textbf{Fashion MNIST: } This is another benchmarking dataset developed as drop in replacement of MNIST dataset. We use a small residual network with two residual blocks and two fully connected layers followed by a softmax layer. NNRELU uses only ReLU activation at all the layers where as SLNN uses SLAF of degree 2 at each layer. We also consider the case where only the final activation is replaced by SLAF (NN\emph{(ReLU+SLAF)}).
        \begin{table}[!htbp]
\centering
\begin{tabular}{|c|c|c|c|}
\hline
\textbf{Algorithm} & \textbf{Degree} & \textbf{Architecture} & \textbf{Acc (\%)} \\ \hline
\textit{NNRELU}    & -                  & 2 Conv + 2 Res. Blocks + 2 Fc                  & 93.56                  \\ \hline
\textit{SLNN}  & $2^8$                  & 2 Conv + 2 Res. Blocks + 2 Fc                  & 92.97                 \\ \hline
\textit{NN(ReLU+SLAF)}  & 2*                  & 2 Conv + 2 Res. Blocks + 2 Fc(ReLU, SLAF)                  & \textbf{93.71}                 \\ \hline
\end{tabular}
\caption{Comparison on Fashion MNIST Dataset with no data augmentation.\\ \emph{* Only describes the product of degrees of SLAFs used and not for the entire NN.}}
\end{table}
\end{enumerate}
\section{Discussion}
We perform experiments on three standard benchmarking datasets---MNIST, FMNIST, and Cifar10 with same architecture but different activations. The NN with all ReLU activations turns ranks second among three approaches chosen for experimentation. On the other hand NN with all SLAF activations gives approximately similar classification accuracies. According to authors, this can be attributed to the optimization difficulties experienced while training SLNNs due to non monotonic and unbounded nature of SLAF. Also, unlike other activations, the polynomials have higher degree terms which causes the input to grow at a very large rate. This results in the exploding activations on test data, which is somewhat minimzed with L2 regularization applied on the network and activation weights. The third method which involves replacing only one ReLU activation with SLAF provides incremental improvement on classification accuracy. Since the SLAF activation is adaptive, the regularization of activation coefficients is observed to minimze overfitting (empirically verified). Note that we don't experiment with the third method on MNIST dataset because the architecture used is much smaller thereby making the optimization of SLNN (all activations replaced) easier and therefore rendering the third one irrelevant. 

One might expect that employing SLAF activation of higher degrees can compensate for more number of layers in the deep neural networks. However, practically we observe that doubling the degree doesn't yield the same performance as adding one layer with SLAF of degree two does. We only provide an intuitive explanation here based on the assumption that increasing the number of global minimas in the parameter space allows SGD to converge to one of them (\cite{globalminima}). Consider a NN with $M$ inputs and with only one hidden layer having $N$ hidden units activated with SLAF of degree four s.t $N>\binom{M+4}{M}$. The output is weighted sum of the activations. If the hidden layer is replaced with two hidden layers each with $N$ hidden nodes having SLAF activation of degree two. The number of extra parameters introduced in the newer architecture would be $N^2$ (much greater than the coefficients for polynomial representation) while the underlying representation is same i.e. a polynomial of degree four. This must mean that any polynomial of degree four must have more than one configuration (such that the resulting polynomial has the same coefficient values) of the SLNN. This implies that the duplicates of the global minima (which will also be a polynomial) are introduced in the newer parameter space thereby making optimization easier.

\section{Conclusion}

We present a new form of activation function which is motivated from polynomial approximation of univariate functions. The activation is learned during training while searching a space of finite degree polynomials. We provide in depth analysis of NNs activated with polynomial activation referred to as SLAF while providing the bounds on the number of parameters of SLNN theoretically required for its underlying polynomial representation. Finally, we show that SLNNs perform at par with standard NNs with experimentation on standard benchmarking datasets. In the end, we provide an intuitive explanation of how different parametrization of SLNNs improve the empirical performance possibly due to properties of SGD algorithm.


\vskip 0.2in
\bibliography{sample}

\appendix
\section{}
\label{app:theorem}


\noindent

\textbf{Assumptions:}
Consider a set \(\textbf{x} = \{x_1, ..., x_n\} \subseteq \mathbb{R}^n\). Let \(\mathcal{B}_k(\textbf{x})\) denote the set of elements of polynomial basis with degree \(k\) constructed using elements of set \(\textbf{x}\). Then,
\begin{equation*}
    \mathcal{B}_k(\textbf{x}) = \left\{\prod_{i=0}^{i=n}x_i^{\alpha_i}\; \middle|\ \; \sum_{i=1}^{i=n}\alpha_i = j, \:j \in \{0, 1, ..., k\} \right\}
\end{equation*}
We call \(\mathcal{B}_k(\textbf{x})\) as the basis set on \(\textbf{x}\) having degree \(k\). We also define monomial set \(\mathcal{M}_k(\textbf{x})\) on \(\textbf{x}\) as the set which contain all monomials of degree \(k\), i.e., 
\begin{equation*}
    \mathcal{M}_k(\textbf{x}) = \left\{\prod_{i=0}^{i=n}x_i^{\alpha_i}\; \middle|\ \; \sum_{i=1}^{i=n}\alpha_i = k\right\}
\end{equation*}
Clearly, \(\mathcal{B}_k(\textbf{x}) = \bigcup\limits_{i=0}^{i=k}\mathcal{M}_i(\textbf{x})\), where \(\bigcup\) denotes the union operator over sets.

\begin{mytheorem}\label{thcp: cardinality of basis}
The cardinality of \(\mathcal{B}_k(\textbf{\emph{x}})\), the set of elements of polynomial basis with degree \(k\) constructed using elements of set \(\textbf{x}\), denoted by \(N_{\mathcal{B}_k}\)  is equal to \(\binom{k + n}{n}\), where \(n\) is the cardinality  of set \(\textbf{\emph{x}}\).

\end{mytheorem}

\begin{proof}\label{proof_thcp: cardinality of basis}
To find the cardinality of the set $\mathcal{B}_k$, consider the following inequality,

\begin{equation}
\label{eq:cardinality}
    0 \leq \alpha_1 + \alpha_2 +\ldots + \alpha_n \leq k
\end{equation}

The cardinality of the set is equal to number of non negative solutions of \eqref{eq:cardinality}. To find the number of solutions, this inequality can be broken down into, $k+1$ equalities, as follows,

\begin{equation}
\label{eq:multi-eq}
    \alpha_1 + \alpha_2 +\ldots + \alpha_n = j \; s.t. \; j \in \{1, 2, \ldots , k\}
\end{equation}

Now, it is straight forward to see that for a fixed j, the number of solutions to \eqref{eq:multi-eq} is equal to $\binom{j + n - 1}{n-1}$, where $\binom{x}{y}$ is equal to the number of ways of choosing $y$ items from $x$ identical items. Now, we can easily write the cardinality as the following summation

\begin{equation}\label{eq:car-step-1}
    N_{\mathcal{B}_k} = \sum_{j=0}^{j=k}\binom{j + n - 1}{n-1}
\end{equation}

Consider the following recurrence relation, which is true $\forall a,b \geq 0$

\begin{equation}
    \binom{a}{b} = \binom{a+1}{b+1} - \binom{a}{b+1}
\end{equation},

We can write \eqref{eq:car-step-1} as, where $F(j) = \binom{j + n - 1}{n}$

\begin{equation}
\begin{split}
    N_{\mathcal{B}_k} &= \sum_{j=0}^{j=k}\binom{j + n - 1}{n-1} = \sum_{j=0}^{j=k} \left( \binom{j + n}{n} - \binom{j + n - 1}{n} \right)\\
    &= \sum_{j=0}^{j=k} (F(j+1) - F(j)) = F(k+1) - F(k)\\
    &= \binom{k + n}{n} - \binom{n - 1}{n} = \binom{k + n}{n}
\end{split}
\end{equation}

\end{proof}

\begin{mytheorem}\label{thcp: monomial cross product}
If \(\mathcal{M}_{k_1}(\textbf{\emph{x}})\) and \(\mathcal{M}_{k_2}(\textbf{\emph{x}})\) denote monomial set on \(\textbf{\emph{x}}\) with degree \(k_1\) and \(k_2\) respectively, then the cartesian product of \(\mathcal{M}_{k_1}(\textbf{\emph{x}})\) and \(\mathcal{M}_{k_2}(\textbf{\emph{x}})\), denoted as \(M_{k_1} \times M_{k_2}\) which contains all possible products of elements from $M_{k_1}$ with elements of set $M_{k_2}$, results in monomial set \(\mathcal{M}_{k_1 + k_2}(\textbf{\emph{x}})\) with degree \(k_1+k_2\).
\end{mytheorem}

\begin{proof}\label{proof_thcp: monomial cross product}
Let the elements of the set \(\mathcal{M}_{k_1}(\textbf{x})\) be denoted by \(x_1^{\alpha_1}...x_n^{\alpha_n}\), elements of the set \(\mathcal{M}_{k_2}(\textbf{x})\) be denoted by \(x_1^{\beta_1}...x_n^{\beta_n}\) and the elements of set \(\mathcal{M}_{k_1+k_2}(\textbf{x})\) be denoted by \(x_1^{\gamma_1}...x_n^{\gamma_n}\). Then, \(\alpha_1+...+\alpha_n = k_1\), \(\beta_1+...+\beta_n = k_2\) and \(\gamma_1+...+\gamma_n = k_1+k_2\). \par
Now, for any \(n-\)tuple \(\{\gamma_1,...,\gamma_n\}\), we can find an \(n-\)tuple \(\{\alpha_1,...,\alpha_n\}\) corresponding to \(\mathcal{M}_{k_1}(\textbf{x})\)  such that \(\gamma_i \geq \alpha_i\; \forall i\). Now, we can choose \(\beta_is\) such that \(\beta_i = \gamma_i - \alpha_i\) and \(\beta_1+...+\beta_n = (\gamma_1 - \alpha_1) +...+(\gamma_n-\alpha_n) = (\gamma_1+...\gamma_n) - (\alpha_1+...\alpha_n) = k_1+k_2 - k_1 = k_2\). Hence, to construct any element \(x_1^{\gamma_1}...x_n^{\gamma_n}\) of the set \(\mathcal{M}_{k_1+k_2}(\textbf{\emph{x}})\), we can find \(x_1^{\alpha_1}...x_n^{\alpha_n}\) and \(x_1^{\beta_1}...x_n^{\beta_n}\) from \(\mathcal{M}_{k_1}(\textbf{x})\) and \(\mathcal{M}_{k_2}(\textbf{x})\) such that their multiplication results in the desired term or cartesian product of \(\mathcal{M}_{k_1}(\textbf{x})\) and \(\mathcal{M}_{k_2}(\textbf{x})\) results in \(\mathcal{M}_{k_1+k_2}(\textbf{x})\).
\end{proof}

\begin{mycorollary}\label{thcp: prod of basis}
If \(\mathcal{B}_{k_1}(\textbf{\emph{x}})\) and \(\mathcal{B}_{k_2}(\textbf{\emph{x}})\) denote basis set on \(\textbf{\emph{x}}\) with degree \(k_1\) and \(k_2\) respectively, then the Cartesian product of \(\mathcal{B}_{k_1}(\textbf{\emph{x}})\) and \(\mathcal{B}_{k_2}(\textbf{\emph{x}})\) results in another basis set \(\mathcal{B}_{k_1 + k_2}(\textbf{\emph{x}})\) with degree \(k_1+k_2\).
\end{mycorollary}

\begin{proof}\label{proof_thcp: prod of basis}
Let $\textbf{K}$ denote the set containing whole numbers less than or equal to $k$, i.e., $\textbf{K} = \{0,\ldots, k\}$. Then,
\begin{equation}
\begin{split}
    \mathcal{B}_{k_1}(\textbf{x}) \times \mathcal{B}_{k_2}(\textbf{x}) &= \bigcup\limits_{i\in K_1}\mathcal{M}_i(\textbf{x}) \times \bigcup\limits_{j\in K_2}\mathcal{M}_j(\textbf{x}) \\
    &=\bigcup\limits_{i,j\in K_1 \times K_2}\left(\mathcal{M}_{i}(\textbf{x}) \times \mathcal{M}_{j}(\textbf{x})\right) \\
    &= \bigcup\limits_{i,j\in K_1 \times K_2} \mathcal{M}_{i+j}(\textbf{x}) \;\;\;\; (Using\; Theorem\; \ref{thcp: monomial cross product})\\
    &= \bigcup\limits_{l \in (K_1+K_2)} \mathcal{M}_{l}(\textbf{x}) \\
    &= \mathcal{B}_{k_1+k_2}(\textbf{\emph{x}})
\end{split}
\end{equation}
\end{proof}

\begin{myclaim}\label{thcp: poly power poly}
Let \(p_{k_1}(\textbf{\emph{x}})\) denotes a polynomial of degree \(k_1\) in \(\textbf{\emph{x}}\). If \(p_{k_1}(\textbf{\emph{x}})\) is transformed by the function  \(f(x) = x^{k_2}\), then the resulting polynomial has a degree \(k_1\cdot k_2\) in \(\textbf{\emph{x}}\).
\end{myclaim}

\begin{proof}\label{proof_thcp: poly power poly}
The function \(f(x)\) can be easily seen as \(x\) multiplied with itself \(k_2\) times.
\begin{equation}
    f(x) = x^{k_2} = x \cdot x \ldots x \cdot x \;\;\text{($k_2$ times)}
\end{equation}
Now, if a polynomial is multiplied with itself it must remain a polynomial in the same input. Therefore, \(f(p_{k_1}(\textbf{\emph{x}}))\) will be a polynomial in \(\textbf{x}\). Now, consider the monomial term in \(p_{k_1}(\textbf{\emph{x}})\) with highest power \(k_1\), when this is multiplied with itself \(k_2\) times, it will results in the power \(k_1\cdot k_2\). Clearly, there can not exist a monomial term with power higher than \(k_1\cdot k_2\). Hence, the degree of polynomial  \(f(p_{k_1}(\textbf{\emph{x}}))\) must be \(k_1\cdot k_2\).
\end{proof}

\begin{mytheorem}
Consider an SLNN with $H$ hidden layers and input denoted as $X_{n \times 1}$ and output as $\textbf{Y}_{m \times 1}$. If the activation at the final layer is linear and all the hidden layers are activated with SLAF of degree $k_i$, where $i$ is the index of the hidden layer. Then, the output of this NN can be reparametrized and written as 
\begin{equation}
\label{eq: nn in matrix}
    \textbf{Y}_{m\times 1} = \textbf{W}_{m\times N_{\mathcal{B}_k}}\textbf{X}^{\mathcal{B}_k}_{N_{\mathcal{B}_k}\times 1}
\end{equation}
where, \(k = \prod_{i=1}^{H}k_i\), called as degree of SLNN, \(\textbf{W}_{m \times N_{\mathcal{B}_k}}\) are the new parameters and \(\textbf{X}^{\mathcal{B}_k}_{N_{\mathcal{B}_k}\times 1}\) is the vector containing polynomial features of degree $\leq k$ in $X_{n \times 1}$. The subscripts in the notation denote matrix size. \\ \textbf{Note:} SLAF with degree equal to one is equivalent to linear/no activation. Therefore, this result directly holds for regression tasks.
\end{mytheorem}

\begin{proof}
As a direct result of claim \ref{thcp: poly power poly}, it is easy to see that any layer of an SLNN can be expressed as a collection of polynomial in SLNN's inputs. Without loss of generality, let the degree of the polynomial obtained as the output of \(i\)th layer be \(d_i\). Now, if degree of SLAF used in \((i+1)\)th is \(k_{i+1}\), then its output will be a polynomial of degree \(d_{i+1} = d_i \cdot k_{i+1}\) (using claim \ref{thcp: poly power poly}). Now, given that $d_0 = 1$, each output node of SLNN denoted by $y_i$ is expressible as a polynomial of degree \(k = \prod_{i=1}^{H}k_i\) and therefore can be reparametrized as
\begin{equation*}
    y_i = \textbf{W}^i_{1\times N_{\mathcal{B}_k}} \textbf{X}^{\textbf{}}_{N_{\mathcal{B}_k} \times 1}
\end{equation*}
Or,
\begin{equation*}
    \textbf{Y}_{m\times 1} = \textbf{W}_{m\times N_{\mathcal{B}_k}}\textbf{X}^{\mathcal{B}_k}_{N_{\mathcal{B}_k}\times 1}
\end{equation*}

where $\textbf{W}_{m\times N_{\mathcal{B}_k}}$ is a matrix with constants which can be easily obtained from the weights of the SLNN.
\end{proof}

\begin{mytheorem}\label{th: nn equivalence}

A neural network with SLAF can approximate any neural network architecture given its input domain is bounded and the activation function $F$ is Lipschitz continuous, to any desired degree of error as a function of degree of SLAF.

\end{mytheorem}

\begin{proof}

First, let us look at Weierstrass Approximation Theorem. It states that for any continuous and real valued function $f(x)$ defined on the interval $[a, b]$, for every $\delta > 0$, there exists polynomial $p(x)$ s.t. for $\forall x \in [a, b]$, we have
\begin{equation}\label{eq: stone weierstrass}
    |f(x) - p(x)| < \delta
\end{equation}

It is also well known that if the function f(x) is not a polynomial, then the degree of the polynomial $p(x)$ approaches infinity as $\delta$ approaches zero. Let's denote the approximation error by $\delta_d$ if the polynomial, $p(x)$ has degree less than equal to $d$. Then for a fixed $f(x)$, on the interval $[a, b]$, it is easy to see that, 

\begin{equation}\label{eq: sw inequality}
    \delta_0 \geq \delta_1 \geq \ldots \geq \delta_\infty
\end{equation}



Consider a neural network with activation function $F(x)$ which follows lipschitz continuity. Let's assume K to be the lipschitz constant for $F(x) : \mathbb{R} \to \mathbb{R}$   s.t. it follows:
\begin{equation}\label{eq: lipschitz cond}
    |F(x_1) - F(x_2)| \leq K|x_1 - x_2|
\end{equation}
The $k^\text{th}$ layer of the NN has $U_k$ hidden units and its linear component be denoted by $h^k(x)$ followed by activation $\phi^k(x)$ which follows:
\begin{gather}
    h^{k+1}_j = \Sigma_{i=1}^{U_k}w^{k}_{ij}\phi^k_i + b^k_j\\
    \phi^{k+1}_j = F(h^{k+1}_j)
\end{gather}

Now consider an Approximate NN (ANN) with all activations replaced by different approximations based on polynomial approximation denoted by $\Tilde{F}$. The $k$th layer with $U_k$ hidden units has a linear component $\Tilde{h}^k(x)$ followed by activation $\Tilde{\phi}^k(x)$ which follows:
\begin{equation}
    \Tilde{h}^{k+1}_j = \Sigma_{i=1}^{U_k}w^{k}_{ij}\Tilde{\phi}^k_i + b^k_j
\end{equation}
\begin{equation}
    \Tilde{\phi}^{k+1}_j = \Tilde{F}(\Tilde{h}^{k+1}_j)
\end{equation}

Now, to get the recursion in error propagated from $k^\text{th}$ layer to $(k+1)^\text{th}$ layer, assume that the approximation error at $k^\text{th}$ layer at activation is upper bounded by $\epsilon_k$ $\forall i \in \{1, \ldots , U_k\}$. Then we can write,
\begin{gather}
    |\phi^{k}_i(x) - \Tilde{\phi}^{k}_i(x)| \leq \epsilon_k\\
    -\epsilon_k \leq \phi^{k}_i(x) - \Tilde{\phi}^{k}_i(x) \leq \epsilon_k
\end{gather}
Now, 
\begin{gather}
    h^{k+1}_j - \Tilde{h}^{k+1}_j = \Sigma_i w^{k}_{ij}(\phi^k_i - \Tilde{\phi}^k_i)\\
    -\Sigma_i |w^{k}_{ij}|\epsilon_k \leq h^{k+1}_j - \Tilde{h}^{k+1}_j \leq  \Sigma_i |w^{k}_{ij}|\epsilon_k \\ \label{eq: hidden layer bound}
    |\Delta h^{k+1}_j| = |h^{k+1}_j - \Tilde{h}^{k+1}_j| \leq \Sigma_i |w^{k}_{ij}|\epsilon_k
\end{gather}

Now, consider $\Delta\phi^{k+1}_j = \phi^{k+1}_j(x) - \Tilde{\phi}^{k+1}_j(x)$,

\begin{equation}
    \Delta{\phi}^{k+1}_j = F(h^{k+1}_j) - \Tilde{F}(\Tilde{h}^{k+1}_j)
\end{equation}

Now from Weierstrass approximation theorem we know that for every $\delta^{k+1}$ there exists a polynomial of degree $d^{j, k}$ $\forall j \in \{1, \ldots , U_{k+1}\}$ denoted by $\Tilde{F}(x)$ which will satisfy $\forall x \in \; I^{j, k+1}_{\text{apprx}} \;: [\min \Tilde{h}^{k+1}_j, \max \Tilde{h}^{k+1}_j]$:
\begin{gather}\label{eq: apprx bound}
    |F(x) - \Tilde{F}(x)| \leq  \delta^{k+1}\\ \label{eq: limit degree}
    s.t. \;\; \lim_{\delta^{k+1} \to 0} d^{j,k} = \infty \;\; (\text{for non polynomial $F(x)$})\\
    h^{k+1}_j - \Sigma_i| w^{k}_{ij}|\epsilon_k = \Sigma_i w^{k}_{ij}\phi^k_i -\Sigma_i| w^{k}_{ij}|\epsilon_k \leq \Tilde{h}^{k+1}_j \leq \Sigma_i w^{k}_{ij}\phi^k_i +\Sigma_i |w^{k}_{ij}| \epsilon_k = h^{k+1}_j + \Sigma_i| w^{k}_{ij}|\epsilon_k
\end{gather}
 The interval over which approximation holds $I^{j, k+1}_{\text{apprx}}$ is easy to calculate for bounded activations. For $\phi \in [\phi_{low}, \phi_{high}]$, we can write
 \begin{gather} \label{eq: bound upper}
     \min \Tilde{h}^{k+1}_j = \Sigma_i w^{k}_{ij}\phi^k_i -\Sigma_i |w^{k}_{ij}| \epsilon_k \geq U_k \cdot \min_i (w^k_{ij}*\phi_{i}) -\Sigma_i |w^{k}_{ij}| \epsilon_k  \\ \label{eq: bound lower}
     \max \Tilde{h}^{k+1}_j = \Sigma_i w^{k}_{ij}\phi^k_i +\Sigma_i |w^{k}_{ij}| \epsilon_k \leq U_k \cdot \max_i (w^k_{ij}*\phi_{i}) + \Sigma_i |w^{k}_{ij}| \epsilon_k
 \end{gather}
 
 These bounds on the interval of approximation provides a method to relate the approximation error with the SLNN and NN attributes. It is easy to see that for a fixed $F(x)$, from eqn. \ref{eq: apprx bound}, \ref{eq: bound upper} and \ref{eq: bound lower}, larger width or weights would expand the interval of approximation. This would mean a higher degree of polynomial would be needed so as to maintain same approximation error.
 
We will now drop the super scripts and sub scripts for sake of clarity. The length of $I_{\text{apprx}}$, depends upon the weights in layer, width of the hidden layer and the range of $\Tilde{\phi}^{k}$. It is easy to see that as the width increases the length of $I_{\text{apprx}}$ increases requiring higher polynomial degree to maintain approximation error. Therefore, the degree of the polynomial in SLNNs acts as a sort of proxy for width in standard NNs. Note that this bound would still hold even if the activation function is unbounded (for ReLU, SeLU, ELU etc) since the input domain is restricted. Now, we can write:
\begin{gather*}
 -\delta^{k+1} \leq F(\Tilde{h}) - \Tilde{F}(\Tilde{h}) \leq \delta^{k+1}\\
 F(h) - F(\Tilde{h}) - \delta^{k+1} \leq F(h) - \Tilde{F}(\Tilde{h}) \leq F(h) - F(\Tilde{h}) + \delta^{k+1} \\
\end{gather*}
But from eq. \ref{eq: lipschitz cond} and \ref{eq: hidden layer bound}, we also 
have

\begin{equation}
     -K|\Delta h| \leq F(h) - F(\Tilde{h}) \leq K|\Delta h|
\end{equation}
Hence, we get
\begin{gather*}
    -K|\Delta h| -\delta^{k+1} \leq F(h) - \Tilde{F}(\Tilde{h}) \leq K|\Delta h| + \delta^{k+1}\\
    |F(h^{k+1}_j) - \Tilde{F}(\Tilde{h}^{k+1}_{j})| = |\phi^{k+1}_j - \Tilde{\phi}^{k+1}_j| \leq K|\Delta h| + \delta^{k+1} 
\end{gather*}
\begin{equation}\label{eq: error bound}
    |\phi^{k+1}_j - \Tilde{\phi}^{k+1}_j| \leq K|\Sigma_i |w^{k}_{ij}| \epsilon_k| + \delta^{k+1} = \epsilon_{k+1}
\end{equation}
    
From eq. \eqref{eq: error bound}, we can see the recursive expression of approximation error is a function of $N$, the degree of polynomial used for approximation $d^{j,k}$. Since, $\epsilon_0 = 0$ (for the input layer), the expression for $\Phi^{k+1}$ would be proportional to $\delta^{k+1}$. This means that by varying $d^{j,k}$, any approximation error can be achieved.
\end{proof}

\end{document}